\pdfoutput=1

\documentclass{article}

\usepackage[final]{neurips_2023}

\usepackage[utf8]{inputenc} 
\usepackage[T1]{fontenc}    
\usepackage[hidelinks]{hyperref}       
\usepackage{url}            
\usepackage{booktabs}       
\usepackage{amsfonts}       
\usepackage{nicefrac}       
\usepackage{microtype}      
\usepackage{xcolor}         

\title{A Logic for Expressing Log-Precision Transformers}

\author{%
  William Merrill \\
  New York University\\
  \texttt{willm@nyu.edu} \\
  \And
  Ashish Sabharwal \\
  Allen Institute for AI \\
  \texttt{ashishs@allenai.org} \\
}

\usepackage{amsmath}
\usepackage{amssymb}
\usepackage{mathtools}
\usepackage{amsthm}

\usepackage[capitalize,noabbrev]{cleveref}

\theoremstyle{plain}
\newtheorem{theorem}{Theorem}
\newtheorem{proposition}{Proposition}
\newtheorem{lemma}{Lemma}
\newtheorem{corollary}{Corollary}[theorem]
\theoremstyle{definition}
\newtheorem{definition}{Definition}

\theoremstyle{remark}

\usepackage[textsize=tiny]{todonotes}

\usepackage{amsmath}
\usepackage{amssymb}
\usepackage{xcolor}
\usepackage{paralist}
\usepackage{graphicx}
\usepackage{thm-restate}

\usepackage{mathtools}
\usepackage{multirow}
\usepackage{bbm}
\usepackage{xspace}
\usepackage{booktabs}
\usepackage{subcaption}

\usepackage{pifont}  
\newcommand{\cmark}{\ding{51}}%
\newcommand{\xmark}{\ding{55}}%

\usepackage{algorithm}
\usepackage[noend]{algpseudocode}

\DeclarePairedDelimiter\ceil{\lceil}{\rceil}
\DeclarePairedDelimiter\floor{\lfloor}{\rfloor}

\DeclarePairedDelimiter\abs{\lvert}{\rvert}%
\DeclarePairedDelimiter\norm{\lVert}{\rVert}%

\makeatletter
\let\oldabs\abs
\def\abs{\@ifstar{\oldabs}{\oldabs*}}
\let\oldnorm\norm
\def\norm{\@ifstar{\oldnorm}{\oldnorm*}}
\makeatother

\theoremstyle{definition}
\newtheorem{example}{Example}

\def\Snospace~{\S{}}

\newcommand{\AC}{\mathsf{AC}}

\newcommand{\TC}{\mathsf{TC}}

\newcommand\FOM{\mathsf{FO(M)}}

\newcommand{\poly}{\mathrm{poly}}

\renewcommand{\O}{\mathrm{O}}

\renewcommand{\O}{\mathrm{O}}

\newcommand{\edge}{\mathsf{edge}}
\newcommand{\node}{\mathsf{node}}
\newcommand{\bnode}{\mathsf{bnode}}
\newcommand{\bstart}{\mathsf{bstart}}
\newcommand{\arity}{\mathsf{arity}}
\newcommand{\size}{\mathsf{size}}
\newcommand{\depth}{\mathsf{depth}}
\newcommand{\bsize}{\mathsf{bsize}}

\newcommand{\calC}{\mathcal{C}}
\newcommand{\calF}{\mathcal{F}}
\newcommand{\calG}{\mathcal{G}}

\newcommand{\calT}{\mathcal{T}}

\newcommand{\bbD}{\mathbb{D}}

\newcommand{\logtimeunif}{log-uniform\xspace}
\newcommand{\logtimecolunif}{log-column-uniform\xspace}

\newcommand{\Q}{Q}
\newcommand{\bit}{\mathsf{bit}}

\newcommand{\col}{K}

\newcommand{\lnorm}{\mathrm{lnorm}}
\newcommand{\sgn}{\mathrm{sgn}}

\newcommand\precision[1]{{\color{blue} #1}}
\renewcommand\precision[1]{#1}

\newcommand{\new}{}

\begin{document}

\maketitle

\begin{abstract}
One way to interpret the reasoning power of transformer-based language models is to describe the types of logical rules they can resolve over some input text. Recently, \citet{chiang2023tighter} showed that finite-precision transformer classifiers can be equivalently expressed in a generalization of first-order logic. However, finite-precision transformers are a weak transformer variant because, as we show, a single head can only attend to a constant number of tokens and, in particular, cannot represent uniform attention. Since attending broadly is a core capability for transformers, we ask whether a minimally more expressive model that can attend universally can also be characterized in logic. To this end, we analyze transformers whose forward pass is computed in $\log n$ precision on contexts of length $n$. We prove any log-precision transformer classifier can be equivalently expressed as a first-order logic sentence that, in addition to standard universal and existential quantifiers, may also contain majority-vote quantifiers. This is the tightest known upper bound and first logical characterization of log-precision transformers.

\end{abstract}

\begin{figure}[h]
    \centering
    \fbox{\begin{minipage}{{0.96\textwidth}}
    \centering
    Any log-precision transformer can be re-expressed as a sentence in $\FOM$ logic,
    e.g.:\vspace{\baselineskip}



    

    \scalebox{1.5}{$\mathsf M i.\ \texttt{a}(i) \ \land \ \mathsf M j.\ \texttt{b}(j) \ \land \ \neg \exists k,\ell.\ (\texttt{a}(k) \land \texttt{b}(\ell) \land \ell < k)$}\vspace{.5\baselineskip}

    \emph{($m$ \texttt{a}'s followed by $m$ \texttt{b}'s, i.e., $\texttt{a}^m\texttt{b}^m$)\vspace{\baselineskip}}

    \scalebox{1.5}{\textcolor[HTML]{009900}{\texttt{aaaabbbb} \cmark}}  \quad \quad 
    \scalebox{1.5}{\textcolor[HTML]{900000}{\texttt{aaabbbbb} \xmark}} \quad \quad
    \scalebox{1.5}{\textcolor[HTML]{900000}{\texttt{baaaabbb} \xmark}}
    \end{minipage}}

    \caption{A first-order logic with majority ($\FOM$) sentence for $\texttt{a}^m\texttt{b}^m$. In addition to standard $\forall$ and $\exists$ quantifiers over string indices, $\FOM$ allows \emph{majority} quantifiers ($\mathsf M$) that take a majority-vote across indices.
    $\texttt{a}(i)$ indicates whether token $i$ is $\texttt{a}$ (and analogously for $\texttt{b}$).
    We prove $\FOM$ can express any function computed by a log-precision transformer.}
    \label{fig:main-figure}
\end{figure}

\section{Introduction}

The incredible success of deep learning models, especially very large language and vision transformers with hundreds of billions of parameters~\citep{LLM:gpt3,LLM:LaMDA}, has come at the cost of increasingly limited understanding of how these models actually work and when they might fail. This raises many concerns, such as around their safe deployment, fairness, and accountability. Does the inner working of a transformer defy description in a simpler symbolic system that we can better understand? Or \emph{can transformer computation be described using a familiar symbolic formalism?} Understanding how to view the reasoning process of a transformer in terms of logic could potentially expand our ability to formally reason about their behavior over large domains of inputs.

\citet{chiang2023tighter} provide a partial answer to this question, showing that any \emph{finite-precision} transformer classifier can be expressed as a sentence in a variant of first-order logic with counting quantifiers and modular arithmetic over input position indices. Specifically, counting quantifiers take the form $\exists^{=x} i : \phi(i)$ where $x$ is a count variable and $i$ is a position index. They show that there exists a single sentence in this logic that computes the output of the transformer for any input string of any length. This is a powerful result because it shows that a simple logical formalism is fully sufficient to describe all the complexity of a massive finite-precision transformer. It also provides an upper bound on finite-precision transformers: any function that cannot be defined in first-order counting logic with modular indexing cannot be expressed by the transformer. 

However, \citeauthor{chiang2023tighter}'s result is not fully general because it relies on the transformer precision being fixed with respect to the transformer's context length. More generally, as we will demonstrate in \Cref{sec:finite-precision}, finite-precision transformers are a fundamentally weak variant of transformers: crucially, cannot express uniform attention patterns, which are a core algorithmic primitive of transformers \citep{weiss-etal-2018-practical}. In fact, we show that they can only attend to a constant number of input positions, which may be seen as a rather limited generalization of hard attention.\footnote{Hard attention is provably substantially weaker than general attention~\citep{angluin2021,merrill2022SatAttnTC0}.}
For example, \citeauthor{chiang2023tighter} show that their logic for finite-precision transformers cannot recognize $\texttt{a}^m\texttt{b}^m$, whereas in practice, transformers can \citep{bhattamishra-etal-2020-ability}.\footnote{Technically, the empirical results of \citet{bhattamishra-etal-2020-ability} are for $\texttt{a}^m\texttt{b}^m\texttt{c}^m$, a harder variant of $\texttt{a}^m\texttt{b}^m$.}
This motivates studying a formal model of transformers where precision grows with context length (which we formalize as \emph{log-precision}), making it possible to capture uniform attention as well as other broad attention patterns. This is useful both for recognizing $\texttt{a}^m \texttt{b}^m$ and more generally for reasoning globally over the input.

We demonstrate that \emph{log-precision} transformer classifiers can also be expressed as sentences in a simple logic: \emph{first-order logic with majority}, or $\FOM$, over inputs strings~\citep{Barrington1988OnUW}. In addition to standard existential and universal quantifiers, $\FOM$ has \emph{majority} quantifiers that return true iff at least half the propositions they quantify are true. It also allows comparing input positions (e.g., $\ell < k$ in \Cref{fig:main-figure}) and accessing their individual bits. Our main result is as follows:

\begin{theorem}[Informal version of \cref{thm:main-fixed}] \label{thm:informal}
    For any log-precision transformer $\calT$, there exists an $\FOM$ sentence $\phi$ that computes the same function as $\calT$, i.e., $\phi(x) = \calT(x)$ for any input string $x$.
\end{theorem}

\paragraph{Upper bound.} \Cref{thm:main-fixed} shows transformers with more than finite precision can also be expressed in a simple extension of first-order logic, going beyond \citet{chiang2023tighter}'s result. On the other hand, $\FOM$ is a strict superset of \citeauthor{chiang2023tighter}'s counting logic; it can simulate counting quantifiers (see \Cref{sec:logic}) and allows non-modular position comparisons. Thus, handling a more general class of transformers powerful enough to express uniform attention slightly weakens the bound.

Still, our result constitutes (to our knowledge) the tightest upper bound on log-precision transformers and the first defined in terms of logic, building on a line of complexity-theoretic work analyzing the power of transformers \citep{hahn-2020-theoretical,merrill2022SatAttnTC0,liu2023transformers,merrill2023parallelism}. In particular, $\FOM$ strengthens the upper bound of log-space-uniform $\TC^0$ by \citet{merrill2023parallelism}.
The refined bound adds to the limitations of transformers identified by \citet{merrill2023parallelism}: for example, it establishes unconditionally that log-precision transformers cannot compute boolean matrix permanents, and shows that, in a certain formal sense, integer division and matching parentheses are among the formally hardest problems that transformers can solve (see \Cref{sec:main-results}).\footnote{To be clear, \Cref{thm:informal} is one-sided: every transformer can be expressed as an $\FOM$ sentence, but not necessarily the other way.
Moreover, we believe that many $\FOM$ sentences \emph{cannot} be expressed by transformers.
An exact logical characterization of transformers remains an open problem.}

\paragraph{Mechanistic interpretability.} Beyond providing an upper bound on the reasoning problems solvable by transformers, we believe \Cref{thm:informal} could guide the design of ``transformer-complete'' programming languages similar in spirit to RASP \citep{weiss-etal-2018-practical}.
RASP is a declarative programming language designed to capture transformer computation, and \citet{tracr} implement a compiler from RASP \emph{into} transformers.
Unlike RASP, $\FOM$ can provably express any transformer (\Cref{thm:informal}), which we believe justifies using it (or an equivalent but more user-friendly variant) as a target language for programs extracted \emph{from} transformers.

Similar to a decision tree, an $\FOM$ sentence has the interpretable property that each sub-sentence corresponds to a constraint on input (see \Cref{fig:main-figure}). In contrast, the internal modules of a transformer or circuit do not satisfy this since they map between arbitrary latent spaces. We speculate this property could facilitate interpreting models by translating them to $\FOM$, though a careful exploration of the algorithmic and HCI aspects of this idea lies outside the current paper's theoretical scope.

\paragraph{Contributions.}
Our results shed new light on how to view the computation inside transformers in terms of logic. Specifically, our main contributions are to prove the following:
\begin{compactenum}
    \item Fixed-precision transformers can only attend to a fixed number of tokens, and those with precision less than $\log \log n$ cannot uniformly attend over length-$n$ contexts (\Cref{prop:precision}).
    \item Log-precision transformer classifiers can be expressed as sentences in $\FOM$ (\Cref{thm:main-fixed}).
\end{compactenum}

\section{Preliminaries: Transformers and \texorpdfstring{$\FOM$}{FO(M)}} \label{sec:preliminaries}

Let $\Sigma$ be a finite alphabet. We denote by $^*$ the Kleene star operator, i.e., for a set $X$, $X^* = \bigcup_{n=0}^\infty X^n$. We will view transformers and $\FOM$ sentences both as functions from $\Sigma^* \to \{0, 1\}$, and show that any function a transformer computes can also be computed by an $\FOM$ sentence.

\subsection{Transformers} \label{sec:transformers-def}

We view the transformer precision $p$ as a function of the context length $n$, writing $p(n)$ where appropriate.
\precision{We assume $p$ to be a power of 2, although our main result can be made to go through even without this assumption \citep{chiang2025transformers}.}
Let $\mathbb D_p$ be the datatype of $p$-precision floats, i.e., tuples $\langle m, e \rangle$ where $m, e$ are signed integers together taking $p$ bits. Using $\abs{x}$ to mean the size of integer $x$, a float represents the value $m\cdot 2^{e - \abs{m} + 1}$.\footnote{$\langle 101, 010 \rangle$ represents $1.01_2 \times 2^{10_2}$. This is closer to the IEEE standard than the $m \cdot 2^e$ semantics used in \citet{merrill2023parallelism}, letting us define the minimum representable float more realistically in \Cref{prop:precision}.}
Following Appendix A of \citet{merrill2023parallelism}, we define $p$-truncated addition ($+, \sum$), multiplication ($\cdot$), and division ($/$) over $\mathbb D_p$.
We now define a \emph{transformer encoder binary classifier} over $\mathbb D_p$, largely adopting \citeauthor{merrill2023parallelism}'s notation.\footnote{Increasing the classifier's output space arity (e.g., a transformer that predicts the next token) or switching to causal attention of a decoder-only model would not change our results. However, our proof no longer goes through if the decoder can generate tokens that get added to the input at the next step \citep[cf.][]{perez2019on}.}

\begin{definition} \label{def:trans}
A $p$-precision transformer $\calT$ with $h$ heads, $d$ layers, model dimension $m$ (divisible by $h$), and feedforward width $w$ is specified by:
\begin{compactenum}
    \item An embedding function $\phi : \Sigma \times \mathbb N \to \mathbb D_p^m$ whose form is defined in \Cref{sec:embedding};\footnote{\label{foot:phi} $\phi$, like $p$, is actually a function of the context length $n$, and \Cref{sec:embedding} enforces that $\phi$ is computable in $\O(\log n)$ time, as standard choices of positional embeddings would satisfy.}
    \item For each $1 \leq \ell \leq d$ and $1 \leq k \leq h$, a head similarity function $s^\ell_k : \mathbb D_p^m \times \mathbb D_p^m \to \mathbb D_p$ whose form is defined in \Cref{sec:self-attention};
    \item For each $1 \leq \ell \leq d$ and $1 \leq k \leq h$, a head value function $v^{\ell}_k : \mathbb D_p^m \to \mathbb D_p^{m/h}$ whose form is defined in \Cref{sec:self-attention};
    \item For each $1 \leq \ell \leq d$, an activation function $f^\ell : (\mathbb D_p^{m/h})^h \times \mathbb D_p^m \to \mathbb D_p^m$ whose form is defined in \Cref{sec:feedforward} and implicitly uses the feedforward dimension $w$;
    \item An output classifier head $\kappa : \mathbb D_p^m \to \{0, 1\}$ whose form is defined in \Cref{sec:output}.
\end{compactenum}
\end{definition}

\begin{definition} \label{def:transformer-computation}
We define the transformer computation and output as a function of an input $x \in \Sigma^n$.
\begin{compactenum}
    \item \underline{Embeddings:} For $1 \leq i \leq n$, $\mathbf h^0_i = \phi(x_i, i)$.$^{\ref{foot:phi}}$
    \item \underline{Self Attention:} For $0 \leq \ell \leq d - 1$, (multihead) self-attention block $\ell + 1$ computes $h$ attention heads:
    \begin{equation*}
        \mathbf a^{\ell+1}_{i,k} = \sum_{j=1}^n \frac{s^{\ell+1}_k(\mathbf h^\ell_i, \mathbf h^\ell_j)}{Z_{i,k}} \cdot v^{\ell+1}_k(\mathbf h^\ell_j), \quad \quad \textrm{where} \; Z_{i,k} = \sum_{j=1}^n s^{\ell+1}_k(\mathbf h^\ell_i, \mathbf h^\ell_j) .
    \end{equation*}
    \item \underline{Activation Block:} For $0 \leq \ell \leq d - 1$, activation block $\ell + 1$ aggregates the head outputs to produce $\mathbf h^{\ell+1}$:
    \begin{equation*}
        \mathbf h^{\ell+1}_i = f^{\ell+1}(\mathbf a^{\ell+1}_{i,1}, \ldots, \mathbf a^{\ell+1}_{i,h}, \mathbf h^\ell_i) .
    \end{equation*}
    \item \underline{Classifier Head:} The network prediction on $x \in \Sigma^n$ is $\kappa(\mathbf h^d_n)$.
\end{compactenum}
\end{definition}

We say $\calT(x) = \kappa(\mathbf h_{\abs{x}}^d)$ and $L_\calT$ is the language of $x \in \Sigma^*$ such that $\calT(x) = 1$.
We refer to $\phi, s^\ell_k, v^\ell_h, f^\ell$, and $\kappa$ as the \textbf{core functions} in $\calT$, and to embeddings, self attention, activation, and the classifier head as the \textbf{components} of $\calT$.
We write $\theta_\calT$ for the concatenated vector of parameters for the functions $\phi, s^\ell_k, v^\ell_h, f^\ell$, and $\kappa$, for all $1 \leq \ell \leq d$ and $1 \leq k \leq h$.

We define a \textbf{log-precision transformer} as one where $p$ is at most $\O(\log n)$ and is a ``simple'' function, i.e., computable in $\O(\log n)$ time.
In our model, the weights $\theta_\calT$ defining $\calT$ are fixed, but the precision $p$ used to compute the forward pass can depend on $n$ (see \Cref{footnote:log-uniform-weights} for a generalization).


\subsection{First-Order Logic with Majority} \label{sec:logic}

As we will show, transformers can be translated into sentences in $\FOM$. But what do such sentences look like?
Informally, $\FOM$ is first-order logic extended to also have majority ($\mathsf{M}$) quantifiers. \new{Following \citet{Barrington1988OnUW}, our sense of $\FOM$ takes \emph{strings} in $\Sigma^*$ as input and returns $0$ or $1$ to define a formal language. In this setting, quantifiers range over \emph{indices} (positions) into the string. Predicates can be applied to the variables introduced by these quantifiers.}

\begin{definition}[$\FOM$ index] \label{def:index}
Indices in $\FOM$ are integers denoting positions in the input string:
\begin{compactenum}

    \item The constant $1$, representing the first token's position.

    \item The constant $n$, representing the last token's position.

    \item Strings (e.g., $i, j, k$) representing variables ranging over positions $1$ to $n$.

    \item Any index built by applying addition or subtraction to other indices.\footnote{\citet{Barrington1988OnUW} did not introduce this as a primitive, but it can be simulated using the $\leq$ predicate.}

\end{compactenum}
\end{definition}

\begin{definition}[$\FOM$ formula] \label{def:formula}
Formulas in $\FOM$ are constructed as follows:\footnote{We write parentheses to indicate the order of operations.}
\begin{compactenum}

    \item Let $\Sigma$ be a finite alphabet. For each $\sigma \in \Sigma$ and any index $i$, $\sigma(i)$, e.g., $\texttt{a}(i)$, is a formula that is true if the $i$-th input token is $\sigma$.\footnote{\citet{Barrington1988OnUW} define $Q_b(i)$ for $b \in \{\texttt 0, \texttt 1\}$. We generalize this to an arbitrary vocabulary $\Sigma$ by assuming each token is one-hot-encoded: \new{$\sigma(i) = \Q_{\texttt 1}(\abs{\Sigma} i + s)$} where $s$ is the index of $\sigma$ in the vocabulary.}

    \item For any indices $i, j$, the formula $\bit(i, j)$ returns the $j$-th bit of the binary expansion of $i$.\footnote{This predicate is included in the logic for technical reasons; see \citet{Barrington1988OnUW}.}

    \item For two indices $i, j$, $i = j$, $i \leq j$, and $i \geq j$ are formulas with their conventional semantics.

    \item For two formulas $\phi, \psi$,$\phi \wedge \psi$ and $\phi \vee \psi$ are formulas with their conventional semantics.

    \item
    For any formula $\phi$ (which may refer to $i$), the following are valid formulas:
    \begin{enumerate}
        \item $\exists i.\ \phi$ means some value of $i$ in $[1, n]$ makes $\phi$ true.
        \item $\forall i.\ \phi$ means all values of $i$ in $[1, n]$ make $\phi$ true.
        \item $\mathsf M i.\ \phi $ means $\geq n/2$ values of $i$ in $[1, n]$ make $\phi$ true.
    \end{enumerate}

\end{compactenum}
\end{definition}

We use parentheses where necessary to disambiguate the order of operations.
\new{General formulas may contain free (i.e., unbound) variables: e.g., $\forall i.\ i = j $.
A \emph{sentence} is an $\FOM$ formula $\phi$ with no free variables.} Sentences represent functions from from $\Sigma^*$ to $\{0,1\}$ and thus define a formal language.\footnote{One can also take multiple sub-sentences within $\phi$ to be labeled as ordered outputs, thus allowing $\phi$ to be a function from $\Sigma^*$ to $\{0,1\}^k$ for some fixed constant $k$.}

\paragraph{Extensions.} Beyond \Cref{def:formula}, $\FOM$ can express \emph{counting} and \emph{threshold} quantifiers in terms of majority quantifiers \citep{Barrington1988OnUW}. Given a formula $\phi$, a counting quantifier creates a new formula $\exists^k i : \phi$ that is true iff $\phi$ is true across exactly $k$ values of $i$. Threshold quantifiers $\exists^{\leq k}$ and $\exists^{\geq k}$ work similarly but check if $\phi$ is true for at least or at most $k$ values of $i$. In addition, we show in \Cref{sec:cond-majority} that $\FOM$ can express \emph{conditional majority} quantifiers, which create a formula $\mathsf M i : \phi \left[ \psi \right]$ that is true iff $\psi$ is true for at least half the values of $i$ that make $\phi$ true.

\subsubsection{Examples}

To illustrate the formalism, we provide example languages definable in $\FOM$ with $\Sigma = \{\texttt{a}, \texttt{b}\}$.
First, we show two languages that do not require majority quantifiers to express:

\begin{example}[Bigram matching] \label{ex:bigram}
Strings containing the bigram $\texttt{ab}$:
$
    \exists i \left[ \texttt{a}(i) \wedge \texttt{b}(i + 1) \right] .
$
\end{example}

\begin{example}[Skip-bigram matching] \label{ex:induction}
Strings containing the long-distance pattern $\texttt{a} \ldots \texttt{b}$ (cf. ``induction heads'' of \citealt{elhage2021mathematical}):
$
    \exists i \left[ \texttt{b}(i) \wedge \exists j \left[ j \leq i \wedge \texttt{a}(j) \right] \right] .
$
\end{example}

In contrast, \Cref{ex:majority} is a simple example that requires majority quantifiers \citep{furst81parity}:

\begin{example}[Majority] \label{ex:majority}
Strings with more $\texttt{b}$'s than $\texttt{a}$'s:
$
    \mathsf{M} i \left[ \texttt{b}(i) \right] .
$
\end{example}

\Cref{fig:main-figure} showed how $\FOM$ can be used to recognize patterns like $\texttt{a}^m\texttt{b}^m$. A similar idea can be used to model parentheses matching \citep{Barrington1988OnUW}:

\begin{example}[$1$-Dyck]
    The well-balanced parentheses language (with \texttt{a} opening and \texttt{b} closing):
    \begin{equation*}
        \forall i.\ ( \exists a, b.\ ( (\exists^a j : \texttt{a}(j) \wedge j \leq i) \wedge (\exists^b j : \texttt{b}(j) \wedge j \leq i) \wedge b \leq a )) \wedge \mathsf M i.\ \texttt{a}(i) \wedge \mathsf M j.\ \texttt{b}(j) .
    \end{equation*}
\end{example}

\begin{example}[Integer Arithmetic]
    Iterated addition (i.e., summing $n$ $n$-bit numbers), iterated multiplication, and division \citep{hesse2001division} can all be expressed in $\FOM$.
\end{example}


\section{Finite Precision Transformers Cannot Attend Universally} \label{sec:finite-precision}


Attention heads that spread attention weight uniformly across inputs have been observed in transformer LMs \citep{merrill2020parameter} and make soft attention fundamentally more powerful than hard attention \citep{angluin2021, merrill2022SatAttnTC0}. In particular, uniform attention is an important primitive that transformers can use to solve tasks involving counting \citep{bhattamishra-etal-2020-ability, chiang2023tighter}, taking majority votes \citep{merrill2022SatAttnTC0}, and matching parentheses or sorting \citep{weiss2021thinking}. A transformer with sufficient precision can easily implement uniform attention by setting the keys and queries across all positions to be constant.
However, attention heads with finite precision cannot represent uniform attention over long sequences as a consequence of the following:
\begin{proposition} \label{prop:precision}
    Let $\mathbf a \in \mathbb R^n$ s.t. $\sum_{i=1}^n a_i = 1$ and $\tilde {\mathbf a}$ its nearest $p$-precision float approximation.
    \begin{compactenum}
        \item  \label{prop:precision-A} Then the number of nonzero entries of $\tilde {\mathbf a}$ is upper bounded by its precision: specifically, $\tilde {\mathbf a}$ has at most
        $2^{2^{p}}$
        nonzero entries.
        \item \label{prop:precision-B} Moreover, if $p < \log \log n$ and $\mathbf a$ is uniform (i.e., $a_i = 1/n$), then $\tilde {\mathbf a} = \vec 0$.
    \end{compactenum}
\end{proposition}
\begin{proof}
    The smallest positive value representable by a $p$-precision float is $2^{-(p_m - 2 + 2^{p_e - 1})}$ which is bounded below by $2^{-2^{p} + 1}$. Letting $k = 2^{2^p}$, it holds that  
        $2^{-2^{p} + 1} = 2 / k$.
    So if $\tilde a_i$ gets the minimum value, then $a_i \geq 1 / k$. Since $\sum_i a_i = 1$, there can be at most $k$ indices satisfying this property. This implies there can be at most $k$ nonzero entries in $\tilde{\mathbf a}$.
    If $n > k$ and $\mathbf a$ is uniform, $1 / n$ is less than half of the minimum representable value of $2/k$. Thus, $\tilde {\mathbf a} = \vec 0$.
\end{proof}

\Cref{prop:precision}  says that fixed-precision transformers are artificially limited because they can only attend over bounded-length windows, making them similar to hard-attention transformers \citep{angluin2021}. Morever, they cannot compute uniform attention over contexts of length $n$ with less than $\log \log n$ precision.
This explains why \citet{chiang2023tighter} prove finite-precision transformers provably cannot recognize $\texttt{a}^m\texttt{b}^m$, while in practice transformers have been shown to learn even its harder variant $\texttt{a}^m\texttt{b}^m\texttt{c}^m$ even with long context lengths \citep{bhattamishra-etal-2020-ability}. In essence, their upper bound only applies in the asymptotic regime when $n > 2^{2^p}$.

In contrast, transformers in practice have enough precision both to compute uniform attention and recognize $\texttt{a}^m\texttt{b}^m$ on practical context lengths.
More concretely, the bfloat16 representation allows uniform attention over $2^{6+2^7} \approx 10^{42}$ tokens and normal float16\footnote{We account for the division of $p$ into $p_m$ and $p_e$ rather than treating them together. Our minimum value differs slightly from numpy but is on the same order of magnitude. Moving to float8 lowers the length upper bound for uniform attention to $2^{3 + 2^3} \approx 2048$, which suggests float8 LMs will have limited length generalization.} allows $2^{10+2^4} \approx 10^8$ tokens, both well above the typical context window of transformers.
This motivates a formal model of transformers with enough precision to compute uniform attention and recognize languages such as $\texttt{a}^m \texttt{b}^m$.




\section{Main Result: Expressing Log-Precision Transformers in \texorpdfstring{$\FOM$}{FO(M)}} \label{sec:main-results}

By \Cref{prop:precision}, precision must grow with the context length $n$ ($p > \log \log n$)
for a transformer to compute uniform attention and other attention patterns with unbounded range, like practical transformers. In this paper, we analyze any transformer with up to $\O(\log n)$ precision. We show that any function computable by log-precision transformers can be expressed in $\FOM$:

\begin{theorem}
\label{thm:main-fixed}
Let $\calT$ be a log-precision transformer with a parameter vector $\theta_\calT$ fixed for all context lengths $n$.\footnote{\label{footnote:log-uniform-weights}\Cref{thm:main-fixed} can also be extended to apply to log-precision transformers with \emph{\logtimeunif weights}, i.e., where $\theta_\calT$ can grow in size and precision with $n$ (see \Cref{sec:omitted-proofs}).
} Then, there exists an $\FOM$ sentence $\phi$ that computes the same function as $\calT$, i.e., $\phi(x) = \calT(x)$ for any input string $x$.
\end{theorem}

\Cref{thm:main-fixed} is the tightest known upper bound for log-precision transformers and shows that it is still possible to characterize transformers in a simple variant of first-order logic even with log-precision and uniform attention.
As alluded to earlier, \Cref{thm:main-fixed} immediately implies that any problem complete for $\FOM$ (or a larger class) is also transformer-hard. Since integer division and Dyck language membership are known to be $\FOM$-complete \citep{hesse2001division,zoo:tc0}, it follows, perhaps surprisingly, that the entire computation of any transformer on input $x$ can be reduced to a single integer division or a finite number of Dyck-language queries:

\begin{corollary}
\label{cor:division}
    Let $\calT$ be a transformer satisfying \Cref{thm:main-fixed}. For any input $x$, there exist first-order definable integers $a, b,$ and $i$ (dependent on $\calT$ and $x$) such that $\calT(x)$ equals the $i$-th bit of $\lfloor a / b \rfloor$. For any $x$, there also exist first-order definable strings $w_1, \ldots, w_m$ such that $\calT(x)$ is first-order definable in terms of the membership of the $w_i$'s in $k$-Dyck. 
\end{corollary}


\section{Preliminaries for Proving Theorem~\ref{thm:main-fixed}} \label{sec:preliminaries2}

\subsection{Computation Graphs}
\label{sec:computation-graphs}

A \emph{computation graph} $G$ over a datatype $\bbD \subseteq \{0, 1\}^*$ and a countable set of primitive functions $\mathfrak F \subseteq \bbD^* \times \bbD$ is a directed acyclic graph where:
\begin{compactenum}
    \item Each node is labelled by a \emph{node type}: a function $f \in \mathfrak F$ computed by this node.
    
    \item Each edge represents a value $\bbD$ flowing as output from one node into another node. We consider the edges flowing into node $j$ to have an order, i.e., be numbered.
    
    \item $\mathfrak F$ contains the special symbol $\mathsf{input}$, which designates $k$ nodes as input nodes. We refer to $k$ as the \emph{arity} and assume w.l.o.g.\ that nodes $0, \ldots, k-1$ are inputs.\footnote{By convention in computer science, we let computation graph nodes be zero-indexed.}
    
    \item A single node is taken as the output node (w.l.o.g., the node with the largest index).
\end{compactenum}

A computation graph $G$ of arity $k$ parameterizes a function $\bbD^k \to \bbD$ in the standard way: the input nodes are assigned the input values, and the value of each node is computed (traversing the graph in a bottom-up topological order) as a function of the values of its children until the output node receives a value. The value of the output node is considered the output of the function. It is worth noting that computation graphs can only process inputs of bounded length. To process arbitrary-length inputs, we will need to generalize them to computation graph families (\Cref{sec:graph-families}).

For a computation graph $G$, $\size(G)$ is the number of nodes, $\depth(G)$ is the length of the longest path from an input node to the output, and $\arity(G, i)$ is the number of inputs to node $i$.

\noindent\textbf{Threshold circuits.} A threshold circuit is a special case of a computation graph where $\bbD = \{0, 1\}$ and $\calF$ is the set of threshold functions of the form $\theta_{\leq \Delta}$ and $\theta_{\geq \Delta}$ over $\bbD^*$, defined as follows: $\theta_{\leq \Delta}(x) = 1$ if $\sum_{\sigma \in x} \sigma \leq \Delta$ and $0$ otherwise; $\theta_{\geq \Delta}(x)$ is defined analogously. Typical AND, OR, and NOT gates are a special case of threshold gates, as is an IDENTITY gate.\footnote{For more background on threshold circuits, see \citet{merrill2023parallelism} and \citet{merrill2022SatAttnTC0}.}

We allow nodes with the $k' \geq 1$ largest indices to all be designated as (ordered) output nodes. A threshold circuit with arity $k$ and $k'$ output nodes will thus be a function from $\{0,1\}^k$ to $\{0,1\}^{k'}$. This will be convenient when simulating neural network components that output multiple bits.

We will find it useful to consider threshold circuits as a kind of compilation target for computation graphs: in other words, we will be concerned with simulating computation graphs defined over more complex functions and data types into threshold circuits.

\subsection{Computation Graph Families} \label{sec:graph-families}

A computation graph family over $\bbD$ and $\mathfrak F$ is a mapping from $n \in \mathbb N$ to a computation graph $G_n$ for processing inputs of size $n$. Thus, $\calG$ defines a function from $\bbD^* \to \bbD$, where $\calG(x) = G_{\abs{x}}(x)$.
Intuitively, computation graph families are useful because they generalize computation graphs to define functions over \emph{unbounded-length} strings as inputs.

\noindent\textbf{Size, depth, and arity.} For computation graph families, the size, depth, and arity become functions of the input length $n$:
$    \size_\calG(n) = \size(G_n), 
    \depth_\calG(n) = \depth(G_n),
    \arity_\calG(n, i) = \arity(G_n, i).
$

\noindent\textbf{Uniformity.} The infinite set $\calG$ can be alternatively represented by two functions:
\begin{compactenum}
    \item $\node_{\calG}(n, i)$, which returns the type of node $i$ in $G_n$ if $i \leq \size(G_n)$, and $\emptyset$ otherwise. For example, if node $i$ computes the logical AND of its inputs, then $\node_{\calG}(n, i) = \wedge$.
    \item $\edge_{\calG}(n, i, j)$, which returns the argument index of $i$ into node $j$ if $G_n$ contains an edge $i \to j$ and $-1$ otherwise. $\edge_{\calG}(n, i, j)$ only needs to be defined over $i,j < \size(G_n)$. For example, if $G_n$ contains a node $j$ with three incoming edges, the second of which comes from node $i$, then $\edge_{\calG}(n, i, j) = 1$.
\end{compactenum}
A pair of algorithms implementing these two functions uniquely specifies a computation graph family, as it enables building the computation graph $G_n$ for any $n$.
Uniform computation graph families (generalizing uniform circuits; cf. \citealp{arora2009computational}) are families where $\node_\calG$ and $\edge_\calG$ can be computed efficiently, i.e., under some constraints on space or time:

\begin{definition}[Uniformity]
A computation graph family $\calG$ is $T(n)$-uniform iff $\node_{\calG}(n, i)$ and $\edge_{\calG}(n, i, j)$ can be computed by a deterministic Turing machine in time $T(n)$.
We focus on \emph{\logtimeunif} computation graph families: i.e., where $T(n) = \O(\log n)$.\footnote{Past work \citep{merrill2023parallelism} analyzes transformers with a similarly named but weaker notion of uniformity, namely log-\emph{space} (rather than log-\emph{time}) uniformity.}
\end{definition}

\noindent\textbf{Threshold circuit families.} These are simply families of threshold circuits. We will be simulating computation graph families with threshold circuit families. Log-uniform $\TC^0$ is the class of languages recognized by \logtimeunif constant-depth, poly-size threshold circuit families. See \citet{merrill2023parallelism,liu2023transformers,arora2009computational} for more background on $\TC^0$ and circuits.


\section{Proof of Theorem~\ref{thm:main-fixed}} \label{sec:main-proof}

The idea is to simulate a transformer with a \logtimeunif $\TC^0$ circuit family.
Since \logtimeunif $\TC^0 = \FOM$, this would imply any transformer can be expressed in $\FOM$.
First, we note that transformers are \logtimeunif computation graphs:

\begin{restatable}[Proof in \Cref{app:transformer-uniform}]{lemma}{lemTransformerUniform} \label{lem:transformers-uniform}
A transformer $\calT$ is a \logtimeunif computation graph family where $\mathfrak F$ contains embedding, self-attention, feedforward, and output components.
\end{restatable}

Further, each core module of the transformer can be simulated by a \logtimeunif $\TC^0$ circuit family:

\begin{restatable}[Proof in \Cref{app:components-uniform}]{lemma}{lemComponentsUniform} \label{lem:components-uniform}
    Let $\calT$ be a log-precision transformer with fixed parameters $\theta_\calT$.
    Then each component in $\mathfrak F$ is computable in \logtimeunif $\TC^0$.
\end{restatable}

Intuitively, we can now simulate a transformer in \logtimeunif $\TC^0$ by just simulating each of its components with a threshold circuit and routing their inputs and outputs appropriately. However, we will need two more technical conditions to verify that this construction is indeed \logtimeunif:

\begin{restatable}[Proof in \Cref{app:components-size}]{lemma}{lemComponentsSize} \label{lem:components-size}
    Let $\calT$ be a log-precision transformer with fixed parameters $\theta_\calT$.
    There exists a function $\bsize(n)$ that is a power of $2$ and computable in $\O(\log n)$ time s.t. $\size_\calF(n) \leq \bsize(n)$ for all $\calF \in \mathfrak F$.
\end{restatable}

\begin{restatable}[Proof in \Cref{app:components-pad}]{lemma}{lemComponentsPad} \label{lem:components-pad}
    If $\calF$ is a \logtimeunif $\TC^0$ family and $\size_\calF(n) \leq \bsize(n)$, there exists a \logtimeunif $\TC^0$ family $\calF'$ s.t. $\calF(x) = \calF'(x)$ for all $x$ and $\size_{\calF'}(n) = \bsize(n)$.
\end{restatable}

Combined, \Cref{lem:components-size,lem:components-pad} show that each $\calF \in \mathfrak F$ is computable by a \logtimeunif $\TC^0$ family with size $\bsize(n)$ that is a power of $2$ and computable in time $\O(\log n)$.
We will show these conditions imply a transformer $\calT$ can be simulated by a $\TC^0$ family $\calC$ (\Cref{thm:nonuniform}) and moreover that $\calC$ is \logtimeunif (\Cref{thm:uniform}).
By the equivalence of \logtimeunif $\TC^0$ and $\FOM$ \citep{Barrington1988OnUW}, we then conclude that any log-precision transformer can be expressed in $\FOM$.

\subsection{Simulating Computation Graph Families with Circuit Families} \label{sec:algorithms}

\new{We give algorithms that take a computation graph family and define a circuit family simulating it.} Intuitively, the algorithms creates contiguous blocks of circuit gates simulating each node in the computation graph and route inputs and outputs between blocks appropriately.

\paragraph{Block mapping.} This algorithm depends on a \emph{block mapping}, which is an implementation of the following three functions:
\begin{compactenum}
    \item The \emph{block node} $\bnode(n, i)$: the index of the node that gate $i$'s block is simulating.
    \item The \emph{block start} $\bstart(n, i')$: the smallest gate index in the block simulating node $i'$.
    \item The \emph{block size} $\bsize(n, i')$: the number of gates in the block simulating node $i'$.
\end{compactenum}
Further, we enforce that a valid block mapping must satisfy that, for all $i$, with $i' = \bnode(n, i)$,
\begin{equation*}
    \bstart(n, i') \leq i < \bstart(n, i') + \bsize(n, i') .
\end{equation*}
Let $\calG$ be a computation graph whose primitive functions are computable by \logtimeunif threshold circuits. We can identify each primitive function with a \logtimeunif threshold circuit family $\calF$ that computes it, where the first $\arity_\calF(n)$ gates are IDENTITY gates reserved for taking input. For such a graph, $\node_\calG$ can be taken to return a symbol identifying a circuit family $\calF$.
In this case, our algorithm requires that, for all $i'$, the block size of $i'$ must match the size of the circuit for the type of block $i'$, i.e., $\bsize(n, i') = \size_{\node_\calG(n, i')}(n)$.
These properties let us meaningfully identify a graph node $i'$ with a block of nodes that will simulate it. This intuition enables us to develop \Cref{alg:node,alg:edge} for constructing a uniform threshold circuit family from a uniform computation graph family.

\begin{figure}[!tp]
    \centering
    \noindent
    \begin{minipage}{0.48\textwidth}
        \begin{algorithm}[H]
        \centering
        \begin{algorithmic}[1]
        \State $\calF \gets \mathsf{node}_{\calG}(n, \bnode(n, i))$
        \If{$\calF \neq \emptyset$}
            \State \Return $\node_\calF(n, i - \bstart(n, i'))$
        \Else{}
            \Return $\emptyset$
        \EndIf
        \end{algorithmic}
        \vspace{3.573cm}
        \caption{\label{alg:node} $\node_{\calC}(n, i)$\\\emph{Return the type of gate $i$ in circuit $C_n$.\vspace{0.075cm}}\\}
        \end{algorithm}
    \end{minipage}
    \hspace{0.25em}%
    \raisebox{-3.3cm}[0cm][0cm]{\rule{0.5pt}{\dimexpr\ht\strutbox+\dp\strutbox+6cm\relax}}%
    \hspace{0.5em}%
    \begin{minipage}{0.48\textwidth}
        \begin{algorithm}[H]
        \centering
        \begin{algorithmic}[1]
        \State $i' \gets \bnode(n, i)$
        \State $j' \gets \bnode(n, j)$
        \State $s_i \gets \bstart(n, i')$
        \State $s_j \gets \bstart(n, j')$
        \If{$i' = j'$}
            \State $\calF \gets \node_{\calG}(n, i')$
            \State \Return $\edge_{\calF}(n, i - s_i, j - s_j)$
        \ElsIf{$(k \triangleq \edge_{\calG}(n, i', j')) \geq 0$}
            \State $b_i \gets i - (s_i + \bsize(n, i') - p(n))$ 
            \State $b_j \gets j - (s_j + kp(n))$ \label{alg-line:p_n_multiplication}
            \If{$b_i = b_j < p(n)$}
                \Return $j - s_j$
            \Else{}
                \Return $-1$
            \EndIf
        \Else{}
            \Return $-1$
        \EndIf
        \end{algorithmic}
        \caption{\label{alg:edge} $\mathcal \edge_{\calC}(n, i, j)$\\\emph{If $C_n$ contains an edge $i \to j$, return the argument number of that edge. Otherwise, return $-1$.}}
        \end{algorithm}
    \end{minipage}
\end{figure}

\begin{theorem} \label{thm:nonuniform}
Let $\calG$ be a computation graph over a finite set of node types $\mathfrak F$, where each $\calF \in \mathfrak F$ is specified by a \logtimeunif circuit family.
Let $\bnode, \bstart,$ and $\bsize$ be a valid block mapping in the sense above.
Then \Cref{alg:node,alg:edge} define a circuit family $\calC$ such that
\begin{compactenum}
    \item $\calC$ and $\calG$ compute the same $\mathbb D_p^* \to \mathbb D_p$ function (let the final $p$ gates of each $C_i$ be its output).
    
    \item $\depth_{\calC}(n) \leq \depth_{\calG}(n) \cdot \max_{\calF} \depth_{\calF}(n)$. 
    
    \item $\size_{\calC}(n) \leq \size_{\calG}(n) \cdot \max_{\calF} \size_{\calF}(n)$.
\end{compactenum}
\end{theorem}

\begin{proof}
Assume w.l.o.g.\ that the gates of $\calC$ are topologically ordered. We show by induction over circuit gates $j$ (with $j' = \bnode(n, j)$) that:
\begin{compactenum}
    \item For all $i' < j'$, the last $p$ nodes of block $i'$ store the value of node $i'$.
    \item For all $i$ such that $\bstart(n, j') \leq i \leq j$, gate $i$ of $\calC$ (as a function of the input nodes of $j'$ ) computes gate $i - \bstart(n, j')$ of $\node_\calG(n, j')$.
\end{compactenum}

\underline{Base case.} We have two circuits with no gates, so the premises are trivially satisfied.

\underline{Inductive case.} Assume the premises hold up to $j$. We will show they hold for $j + 1$.
Let $\calT = \node_\calG(n, j')$.
By Premise 1, we know that the last $p$ nodes of block $i'$ store the output of node $i'$, for $i' < j'$. By \Cref{alg:edge},
for each $i'$ such that $\edge_\calG(n, i', j') = k$ with $0 \leq k < \arity_\calF(n)$,
gates $kp$ through $(k + 1) p - 1$ of block $j'$ will copy the final $p$ gates of block $i'$.
Thus, the first $k \times \arity_\calF(n)$ gates of block $j'$ store the inputs to node $j'$.

At this point, we use Premise 2 to conclude that the first $j - \bstart(n, j')$ gates of block $j'$ compute the same function as the first $j - \bstart(n, j')$ gates of $\calF$ with respect to this input.
Thus, we just need to show that gate $j + 1$ is also correct.
Within \Cref{alg:edge}, we fall in case $i' = j'$, meaning that gate $j + 1$ of block $j'$ gates the same inputs as gate $j + 1$ of $\calF$.
By \Cref{alg:node}, the type of gate $j + 1$ in block $j'$ is the type of gate $j + 1$ of $\calF$.
Thus, gate $j + 1$ in block $j'$ computes the same function of the input gates as gate $j + 1$ in $\calF$.
If $j + 1 = \bsize(n, j')$, we conclude that the final $p$ gates of block $j'$ store the output of node $j'$.
\end{proof}

Let $\mathsf{XC}^0$ denote any family of constant-depth, poly-size circuits, including $\AC^0$ and $\TC^0$.\footnote{\new{Formally, $\mathfrak F$ just needs to contain $\wedge$ and $\vee$.}}

\begin{corollary}
\label{cor:comp-graphs-as-XC-circuits}
Let $\calG$ be a \new{constant-depth, poly-size} computation graph family over a finite $\mathfrak F$. If every node type in $\mathfrak F$ can be computed by $\mathsf{XC}^0$ circuits, the function computed by $\calG$ is in $\mathsf{XC}^0$.
\end{corollary}

Since a transformer has constant depth and polynomial size, \Cref{cor:comp-graphs-as-XC-circuits} lets us easily recover prior results about hard-attention transformers~\citep{angluin2021,hahn-2020-theoretical} and saturated attention transformers~\citep{merrill2022SatAttnTC0} using a common framework. All one has to do is show that all individual node types in such transformers can be computed by $\AC^0$ and $\TC^0$ circuits, respectively.


\new{\Cref{cor:comp-graphs-as-XC-circuits} established that \Cref{alg:node,alg:edge} construct a circuit family that simulates $\calG$.}
\new{With the right block mapping, $\calC$ will be \logtimeunif as long as $\calG$ and its node types are \logtimeunif.}

\begin{corollary} \label{thm:uniform}
Let $\calG$ be a \logtimeunif, constant-depth computation graph family over a finite $\mathfrak F$, where each $\calF \in \mathfrak F$ is specified by a \logtimeunif $\TC^0$ family
with $\size_\calF(n) = \bsize(n)$ that is a power of $2$ computable in $\O(\log n)$ time. 
Then $\calG$ can be simulated by a \new{\logtimeunif} $\TC^0$ family $\calC$ that obeys the size and depth properties of \Cref{thm:nonuniform}.
\end{corollary}

\begin{proof}
\new{Let $\calC$ be the circuit family defined by \Cref{alg:node,alg:edge} given $\calG$ and the following block mapping:}
$
    \bnode(n, i) = \floor{i / \bsize(n)},
    \bstart(n, i') = i' \cdot \bsize(n),
    \bsize(n, i') = \bsize(n).
$
\new{Since $\bsize(n)$ is a power of $2$, $\bnode$ and $\bstart$ are reducible to left and right shifting over $\O(\log n)$-bit integers,}
which can be implemented in $\O(\log n)$ time.
Thus, \new{each block mapping function is computable in time $\O(\log n)$}. \precision{Similarly, as $p(n) = 2^m$ for some $m$ (cf.~\Cref{sec:transformers-def}), $k p(n)$ in line~\ref{alg-line:p_n_multiplication} of $\edge_\calG$ can be implemented in $\O(\log n)$ time by left-shifting $k$ by $m$.
Finally, since $\node_\calG$ and $\edge_\calG$ are simply} calling \new{functions computable in time $\O(\log n)$} with a constant overhead,
\new{we conclude that $\calC$, the circuit family they define, is \logtimeunif. From \Cref{thm:nonuniform}, $\calC$ is already known to simulate $\calG$ with constant depth and polynomial size, completing the proof.}
\end{proof}

\section{Conclusion}

We proved that any log-precision transformer classifier can be translated to an $\FOM$ sentence that computes the same function (on all inputs of any length).
This result comes by first simulating a transformer with a highly uniform threshold circuit family, and then leveraging the established equivalence of \logtimeunif circuits and $\FOM$.
Transformers and other neural nets are often discussed in contrast with symbolic models based on logical formalisms \citep{GARNELO201917}---an immediate implication of our result is that it is possible to express the inner workings of transformers also in a simple logic, challenging the premise of a rigid division between symbolic and neural models.
Our results also provide the tightest known upper bound on log-precision transformers.


While it is striking that a full transformer can be translated to a sentence in a logic as simple as $\FOM$, we believe the bound is not tight.
In particular, we conjecture that it is possible to simulate any transformer with an $\FOM$ sentence of quantifier depth of at most 2, which could be proven by establishing a hierarchy theorem describing the $\FOM$ quantifier depth needed to simulate a $\TC^0$ family of a certain size.
It would also be an interesting extension to translate real transformers to $\FOM$ sentences. In this sense, we believe our results provide a theoretical foundation to guide mechanistic interpretability work \citep[cf.][]{weiss2021thinking, tracr}.

Our findings provide a novel view into transformer classifiers and their limits. It would be exciting for future research to extend our results to account for other common practical uses of transformers, such as for long-form generation, chain-of-thought reasoning, and in-context learning.


\subsubsection*{Acknowledgments}
We thank Paul Beame, David Chiang, anonymous reviewers, and researchers at the Allen Institute for AI for feedback.
Thanks to Noa Nabeshima and Kai Yee for identifying minor issues that have been corrected.
WM was supported by an NSF graduate research fellowship and in part by NSF award 1922658.


\bibliography{references}
\bibliographystyle{icml2023}

\newpage
\appendix
\section{Conditional Majority} \label{sec:cond-majority}

Given formulas $\phi, \psi$, $\mathsf M i : \phi.\ \psi$ is a sentence that is true iff $\psi$ is true for at least half the values of $i$ that make $\phi$ true.

\begin{proposition}
For any two predicates $\phi(i)$ and $\psi(i)$, $\mathsf M i : \phi(i) .\ \psi(i)$ can be expressed in $\FOM$.
\end{proposition}

\begin{proof}
$\mathsf M i : \phi .\ \psi $ can be rewritten using a counting quantifier and a threshold quantifier:
\begin{equation*}
    \exists k, k' .\ \left[ 2k' = k \wedge  \exists^k i : \phi(i) \wedge \exists^{\geq k'}j : \left(\phi(j) \land \psi(j) \right)\right] .
\end{equation*}
The formula $2k' = k$ can be defined using $\bit$.
We then use the fact that counting and threshold quantifiers can be expressed in terms of majority quantifiers \citep{Barrington1988OnUW} to conclude that $\mathsf M i : \phi .\ \psi$ can be expressed in $\FOM$. 
\end{proof}

\section{Omitted Proofs} \label{sec:omitted-proofs}

\Cref{table:notation} summarizes the notation we use in the following proofs when describing computation graphs and circuit families.

\begin{table}[ht]
\caption{\label{table:notation} Summary of common notation for computation graph and circuit families.}
\centering
\begin{tabular}{|cccl|}
    \hline
    Graph & Circuit & Output Range & Description \\
    \hline
    $i'$ & $i$ & $\mathbb Z$ & index of node or gate \\
    $\node_{\calG}(n, i')$ & $\node_{\calC}(n,i)$ & $\mathfrak F$\footnote{We abuse notation and consider the node type of a computation graph whose primitive functions are computable by circuit families to be those circuit families.} & type of node or gate \\
    $\edge_{\calG}(n,i', j')$ & $\edge_{\calC}(n,i, j)$ & $\mathbb Z$ & argument \# of edge $i \to j$ \\
    $\size_{\calG}(n)$ & $\size_{\calC}(n)$ & $\mathbb Z$ & \# of nodes or gates\\
    $\depth_{\calG}(n)$ & $\depth_{\calC}(n)$ & $\mathbb Z$ & longest path length\\
    \hline
    \multicolumn{4}{c}{} \\ \hline
    \multicolumn{2}{|c}{$\bnode(n, i)$} & $[0, \size_\calG(n)]$ & block containing $i$ \\
    \multicolumn{2}{|c}{$\bstart(n, i')$} & $[0, \size_\calC(n)]$ & first gate in block $i'$ \\
    \multicolumn{2}{|c}{$\bsize(n, i')$} & $\mathbb Z$ & size of block $i'$ \\
    \hline
\end{tabular}
\end{table}

\subsection{Transformers are Log-Uniform Computation Graph Families} \label{app:transformer-uniform}

\new{We now justify that the computation graph family defining a transformer is \logtimeunif.}
To do this, we introduce a stronger notion of uniformity called \emph{column uniformity} that captures the highly regular structure of the transformer.

\new{Let $\node(G, i)$ be the $i$-th node of computation graph $G$. Let $a \bmod b$ be the remainder when $a$ is divided by $b$.}
\begin{definition}[Column uniformity]
A computation graph family $\calG$ is $T(n)$-column-uniform iff \new{there exists a computation graph $\col$ (with fixed size w.r.t $n$) such that, for all $i,j$ such that $0 \leq i, j < \size_\calG(n)$:}
\begin{compactenum}
    \item \new{$\node_\calG(n, i) = \node \left( \col, i \bmod \size(\col) \right)$}.
    \item If $\floor{i / \size(\col)} = \floor{j / \size(\col)}$, then
    \begin{equation*}
        \edge_{\calG}(n, i, j) = \edge \left( \col, i \bmod \size(\col), j \bmod \size(\col) \right) .
    \end{equation*}
    Otherwise, $\edge_{\calG}(n, i, j)$ can be computed by a deterministic Turing machine in time $T(n)$.
\end{compactenum}
\end{definition}

We define \emph{\logtimecolunif} analogously to \logtimeunif: i.e., we let $T(n) = \O(\log n)$. \new{\logtimecolunif implies \logtimeunif because our implementations of $\node_\calG$ and $\edge_\calG$ can store $K$ in a finite lookup table and compute the quotient and remainder of $i$ and $j$ by $\size(K)$ in $\O(\log n)$ time using \Cref{lem:division}. The edges outside of $K$ are computable in $\O(\log n)$ time by construction.}

\lemTransformerUniform*

\begin{proof}
We show the stronger condition that any transformer $\calT$ is a \logtimecolunif computation graph family, which implies it is \logtimeunif.

We have the column $\col$ by \Cref{def:transformer-computation}:
all that remains to show is that $\edge_{\calG_\calT}$ can be computed in time $\O(\log n)$ \new{for edges outside the column}. These edges route from the layer $\ell$ output to the self-attention heads of layer $\ell + 1$.
\new{Following from the column structure, there exists $k_\ell$ such that a node $i$ is an output vector of layer $\ell$ iff $k_\ell = i \bmod \size(K)$.
In a finite lookup table, we can store $k_\ell$ for each $\ell + 1$, and use this for self-attention routing.
For an unmasked self-attention head $j$, we compute:
\begin{equation*}
    \edge_{\calG_\calT}(n, i, j) = \begin{cases}
        \floor{i / \size(K)} & \textrm{if} \; k_\ell = i \bmod \size(K) \\
        -1 & \textrm{otherwise.}
    \end{cases}
\end{equation*}}
For causally masked attention, \new{we extend the first case to check that $\floor{i/\size(K)} \leq \floor{j/\size(K)}$}. Either way, this logic \new{can be implemented in time $\O(\log n)$ via \Cref{lem:division}}. Thus, we conclude that \new{$\calG_T$} is column-uniform.
\end{proof}

\subsection{Transformer Components are Computable by Log-Uniform Threshold Circuits} \label{app:components-uniform}

\lemComponentsUniform*

We prove a more general version of \Cref{lem:components-uniform} that handles some cases with weights growing with $n$.
The weights $\theta_\calT$ are just a special case of a computation graph (that do not depend on the input); we can thus apply our definition of \logtimeunif to them.
\Cref{lem:components-uniform} follows from a more general result with \logtimeunif $\theta_\calT$:

\begin{lemma}
Let $\calT$ be a \logtimeunif transformer with \logtimeunif $\theta_\calT$.
Then each component in $\mathfrak F$ is computable in \logtimeunif $\TC^0$.
\end{lemma}

\begin{proof}
In \Cref{sec:modules}, we show that \logtimeunif $\theta_\calT$ implies:
\begin{compactenum}
    \item The embedding component is computable in \logtimeunif $\TC^0$ (\Cref{lem:embedding}).
    \item The self attention mechanism is computable in \logtimeunif $\TC^0$ (\Cref{lem:self-attention}).
    \item The activation block is computable in \logtimeunif $\TC^0$ (\Cref{lem:feedforward}).
    \item The output classifier head is computable in \logtimeunif $\TC^0$ (\Cref{lem:output}).
\end{compactenum}
We have shown that each $\calF \in \mathfrak F$ is computable in \logtimeunif $\TC^0$.
\end{proof}

\subsection{Transformer Component Size Has a Log-Time Upper Bound} \label{app:components-size}

\lemComponentsSize*

\begin{proof}

Let $2^{b(n)}$ be the least power of $2$ at least as large as $\size_\calF(n)$ for all $\calF$.
We observe that $2^{b(n)}$ is at most $2 \cdot \max_\calF \size_\calF(n)$ for all $n$. Because each $\calF$ has poly size, there is a fixed $k$ such that, for large enough $n$,\footnote{We can compute $\bsize(n)$ for small $n$ using finite lookup.}
\begin{align*}
    2^{b(n)} &\leq n^k \\
    \Rightarrow b(n) &\leq k \ceil{\log n} .
\end{align*}
Define $b'(n) = k \ceil{\log n}$ and $\bsize(n) = 2^{b'(n)}$. $\bsize(n)$ is both a power of $2$ and an upper bound on $2^{b(n)}$; what remains to be shown is that it can be computed in time $\O(\log n)$.
We can first compute $\ceil{\log n}$ in time $\O(\log n)$ by finding the greatest nonzero index of $n$.
Next, we can compute $b'(n) = k \cdot \ceil{\log n}$ in time $\O(\log \log n)$ since $k$ is fixed size and $\ceil{\log n}$ has size at most $\O(\log \log n)$ \citep{brent2010modern}.
Finally, we compute $\bsize(n) = 2^{b'(n)}$ by simply left-shifting $1$ at most $\O(\log n)$ times.
\end{proof}

\subsection{Circuit Families Can Be Padded to Log-Time Size Upper Bounds} \label{app:components-pad}

Recall that the last $p$ bits of our circuits represent the circuit's output (cf.~\cref{sec:computation-graphs}).
In \Cref{lem:components-pad}, we consider $\calF(x) = \calF'(x)$ if and only if the last $p$ bits of $\calF$ and $\calF'$ agree for all $x$.

\lemComponentsPad*

\begin{proof}
    The high level idea is that we can pad $\calF$ to a circuit $\calF'$ that has size $\bsize(n)$ and simply copies over the $p$ output bits of $\calF$ to its own last $p$ bits using identity gates.

    We first set $\node_{\calF'}$ to copy over the existing circuit and append identity nodes.
    Let $\textrm{Id}$ denote an identity node. Then $\node_{\calF'}$ is defined as:
    \begin{equation*}
        \node_{\calF'}(n, i) = \begin{cases}
            \node_{\calF}(n, i) & \textrm{if} \; \node_{\calF}(n, i) \neq \emptyset \\
            \textrm{Id} & \textrm{if} \; \node_\calF(n, i) = \emptyset \wedge i < \bsize(n) \\
            \emptyset & \textrm{otherwise} .
        \end{cases}
    \end{equation*}
    We see that the size of $\mathcal F'$ will thus be of size $\bsize(n)$.
    
    Next, we extend $\edge_{\calF'}(n, i, j)$ to route the original output bits to the new output bits.
    Recall that an edge value of $0$ means $i$ is the first argument of gate $j$, and an edge value of $-1$ means there is no edge $i \to j$.
    Let $k_j = p(n) - (\bsize(n) - j)$ be the index of node $j$ as an output gate in $\calF'$. For example, $k=0$ for the first output bit.
    Now let $\mathsf{output}_\calF(n, i, k)$ represent whether node $i$ is the $k$-th output of $F_n$. We can compute $\mathsf{output}_\calF(n, i, k)$ in terms of $\node_\calF$ as follows:
    \begin{equation*}
        \mathsf{output}_\calF(n, i, k) \iff \node_{\calF}(n, i + p(n) - k - 1) \neq \emptyset \wedge \node_\calF(n, i + p(n) - k) = \emptyset .
    \end{equation*}
    Then $\edge_{\calF'}$ is defined:
    \begin{equation*}
        \edge_{\calF'}(n, i, j) = \begin{cases}
            \edge_{\calF}(n, i, j) & \textrm{if} \; \edge_{\calF}(n, i, j) \neq -1 \\
            0 & \textrm{if} \; \mathsf{output}_\calF(n, i, k_j) \\ 
            -1 & \textrm{otherwise} .
        \end{cases}
    \end{equation*}
    The first condition simply copies over the original edges. The second condition adds $p(n)$ new edges (for the different values of $k$) that route the final $p(n)$ nodes of $\calF$ to the final $p(n)$ nodes of $\calF'$, guaranteeing that the two circuits will compute the same function.

    Because both $\node_{\calF'}$ and $\edge_{\calF'}$ just rely on addition, conditional branching, and a finite number of calls to functions computable in time $\O(\log n)$, they are both computable in time $\O(\log n)$.
\end{proof}

\section{Transformer Column Components}
\label{sec:modules}

In this section, we generally omit layer subscripts for clarity.
We assume a pre-norm \citep{xiong2020on} parameterization of the transformer for concreteness and because this is more standard in newer transformers.
However, the results would also hold with the original post-norm \citep{vaswani2017attention}.

As mentioned in the main text, we view $\theta_\calT$ as a concatenation of the parameters for the transformer functions. Thus, if $m$ and $w$ are computable in time $\O(\log n)$ and $\theta_\calT$ is \logtimeunif, it follows that the parameter vector for each $\phi, s, v, f$, and $\kappa$ is itself \logtimeunif because we can map indices in the smaller parameter vectors to indices in $\theta_\calT$ in time $\O(\log n)$.

\subsection{Transformer Embeddings} \label{sec:embedding}

For each position $1 \leq i \leq n$,
the transformer embedding function represents token $\sigma_i \in \Sigma$ and its position $i$ with a vector.
Let $\mathbf V$ be an embedding matrix of size $\abs{\Sigma} \times m$ where each row represents the embedding for some $\sigma$.
Let $f : \mathbb N \to \mathbb D_p^m$ be computable in time $\O(\log n)$.
Then,
\begin{equation*}
    \mathbf \phi(\sigma_i, i) = \mathbf v_{\sigma_i} + f(i) .
\end{equation*}

\begin{lemma} \label{lem:embedding}
If $\theta_\calT$ is \logtimeunif,
then $\phi$ is computable in \logtimeunif $\TC^0$.
\end{lemma}

\begin{proof}
The embedding block can be expressed as a constant-size computation graph that constructs $\mathbf V$, computes $\mathbf v_{\sigma_i}$ using an affine transformation, computes $f(i)$, and then, finally, sums $\mathbf v_{\sigma_i}$ and $f(i)$. The first step is computable by a \logtimeunif constant-depth, poly-size threshold circuit family since $\theta_\calT$ is \logtimeunif. We can compute an affine transformation via a \logtimeunif constant-depth poly-size threshold circuit family via \Cref{lem:affine}. $f(i)$ can be directly computed by the Turing machine constructing the circuit by construction. The sum of the two terms can then be computed by a \logtimeunif constant-depth threshold circuit of size polynomial in $m$, which is also polynomial in $n$. Since we have a computation graph where all node types are computable by \logtimeunif, constant-depth, poly-size threshold circuit families, we conclude by \Cref{thm:uniform} that $\phi$ can also be computed by \logtimeunif, constant-depth, poly-size threshold circuit family.
\end{proof}

\subsection{Self Attention} \label{sec:self-attention}

The two components of the self attention block are $s$, the similarity function, and $v$, the value function.
Let $\mathbf h_i$ be the hidden state at the previous layer and $\bar {\mathbf h}_i = \lnorm(\mathbf h_i)$.
Then, the similarity function first computes queries and keys, and then takes the scaled dot-product between them:
\begin{align*}
    \mathbf q_i &= \mathbf W_q \bar {\mathbf h}_i + \mathbf b_q \\
    \mathbf k_i &= \mathbf W_k \bar {\mathbf h}_i + \mathbf b_k \\
    s(\mathbf h_i, \mathbf h_j) &= \exp \left( \frac{\mathbf q_i^\top \mathbf k_i}{\sqrt{m/h}} \right) .
\end{align*}
Then the value function is defined $v(\mathbf h_i) = \mathbf W_h \bar {\mathbf h}_i + \mathbf b_h$.
We first show that the value function (and also the keys and queries by symmetry) is computable in \logtimeunif $\TC^0$:

\begin{lemma} \label{lem:self-attention}
If $\theta_\calT$ is \logtimeunif, then the self-attention component is computable in \logtimeunif $\TC^0$.
\end{lemma}

\begin{proof}
$v$ is a composition of constructing the parameters (in \logtimeunif $\TC^0$ since $\theta_\calT$ is \logtimeunif), layer norm (in \logtimeunif $\TC^0$ by \Cref{lem:layer-norm}), and an affine transformation (in \logtimeunif $\TC^0$ by \Cref{lem:affine}). Thus, $v$ is computable in \logtimeunif $\TC^0$.

Computing $s$ is a constant-depth computation graph. First, we compute $\mathbf q_i$ and $\mathbf k_i$ and then multiply them, and all of these steps are in \logtimeunif $\TC^0$. Next, we can compute $m$ and $h$ in time $\O(\log n)$ and build a \logtimeunif $\TC^0$ circuit that divides the product of the last step by $\sqrt{m/h}$.
Finally, we compute $p$-precision $\exp$, which can be expressed in \logtimeunif $\TC^0$ as multiplication followed by left-shifting.
Thus, by \Cref{thm:uniform}, $s$ can be computed in \logtimeunif $\TC^0$.

$s$ and $v$ are \logtimeunif, so their size $p$ is at most $\poly(n)$. Computing self attention reduces to binary multiplication and division over $\mathbb D_p$, and performing iterated addition (summation) over $n$ numbers in $\mathbb D_p$. Binary multiplication, binary division \citep{hesse2001division}, and iterated addition \citep{merrill2023parallelism} can all be computed in \logtimeunif $\TC^0$, i.e., by a \logtimeunif, constant-depth threshold circuit family of size at most $\poly(p) \leq \poly(n)$. Thus, self attention can also be computed in \logtimeunif $\TC^0$.
\end{proof}




\subsection{Activation Block} \label{sec:feedforward}

The activation function $f$ encapsulates the aggregation of the attention head outputs and the feedforward subnetwork of the transformer. $f$ takes as input attention head outputs $\mathbf a_{i,1}, \ldots, \mathbf a_{i,h} \in \mathbb D_p^{m/h}$ and the previous layer value $\mathbf h_i$.

The first part of the activation block simulates the pooling part of the self-attention sublayer.
The head outputs are first concatenated to form a vector $\mathbf a_i$, which is then passed through an affine transformation $(\mathbf W_o, \mathbf b_o) : \mathbb D_p^m \to \mathbb D_p^m$ followed by residual connections to form the sublayer output $\mathbf o_i \in \mathbb D_p^m$:
\begin{equation*}
    \mathbf o_i = \mathbf W_o \mathbf a_i + \mathbf b_o + \mathbf h_i .
\end{equation*}

The second part of the activation block first applies layer-norm and then simulates the feedforward subnetwork to compute the next layer vector $\mathbf h'_i$.
Let $\bar{\mathbf o}_i = \lnorm(\mathbf o_i)$.
Let $\sigma$ be a nonlinearity computable in linear time on its input (in the most standard transformer, ReLU). Then, for affine transformations $(\mathbf W_1, \mathbf b_1) : \mathbb D_p^m \to \mathbb D_p^w$ and $(\mathbf W_2, \mathbf b_2) : \mathbb D_p^w \to \mathbb D_p^m$, the feedforward subnetwork can be defined:
\begin{align*}
    \mathbf h'_i = \mathbf W_2 \sigma(\mathbf W_1 \bar{\mathbf o}_i + \mathbf b_1) + \mathbf b_2 + \mathbf o_i .
\end{align*}

\begin{lemma} \label{lem:feedforward}
If $\theta_\calT$ is \logtimeunif, then $f$ is computable in \logtimeunif $\TC^0$.
\end{lemma}

\begin{proof}
The activation block can be expressed as a constant-size computation graph where the nodes construct affine transformation parameters, apply affine transformations, compute layer-norm, and compute elementwise nonlinearities. Since each of these nodes is computable by a \logtimeunif, constant-depth, poly-size threshold circuit family, the activation block is as well.
\end{proof}

\subsection{Output Classifier Head} \label{sec:output}

We assume the output from the transformer is computed as follows. First, $\bar {\mathbf h}_1 = \lnorm(\mathbf h_1)$. Then, we use a parameter vector $\mathbf w \in \mathbb D_p^m$ and bias term $b$ to compute:
\begin{equation*}
    \kappa(\mathbf h_1) = \sgn(\mathbf w^\top \bar {\mathbf h}_1 + b) .
\end{equation*}

\begin{lemma} \label{lem:output}
If $\theta_\calT$ is \logtimeunif, then $\kappa$ is computable in \logtimeunif $\TC^0$.
\end{lemma}

\begin{proof}
We can express computing $\kappa$ as a composition of constructing the parameters $\mathbf w, b$ and computing the affine transformation. Both parts of this composition are computable by a \logtimeunif, constant-depth, poly-size threshold circuit family, so computing $\kappa$ is as well.
\end{proof}

\section{Neural Net Building Blocks}

In this section we analyze the uniformity of common neural net building blocks that are used within the various high-level transformer components.







\subsection{Affine Transformations}

Affine transformations are a core part of neural networks used in various parts of the transformer. An affine transformation takes as input parameters $(\mathbf W, \mathbf b) : \mathbb D_p^a \to \mathbb D_p^b$ and a vector $\mathbf x \in \mathbb D_p^a$ and returns $\mathbf W \mathbf x + \mathbf b$.

\begin{lemma} \label{lem:affine}
For $p = \O(\log n)$, any $p$-precision affine transformation where $\mathbf W, \mathbf b$ are \logtimeunif is computable by a \logtimeunif, constant-size threshold circuit family of size polynomial in $a$ and $b$.
\end{lemma}
\begin{proof}
We first use the uniformity of $\mathbf W, \mathbf b$ to construct them in $\O(\log n)$ time.
For the transformation $\mathbf W \mathbf x + \mathbf b$, first compute each $\mathbf w_i \odot \mathbf x$ in parallel, where $\odot$ represents elementwise multiplication. Since binary multiplication over polynomial-size numbers is in \logtimeunif $\TC^0$, this can be done in parallel with \logtimeunif $\TC^0$ circuits. We then use $b$ \logtimeunif, constant-depth, poly-size threshold circuit families, each corresponding to an output index, that compute the sum over the $a$ entries of each $\mathbf w_i \odot \mathbf x$. The affine transformation corresponds to the composition of these two steps, and is thus computable by a \logtimeunif $\TC^0$ circuit family.
\end{proof}

\subsection{Layer Norm}

The layer norm is applied between sublayers in the transformer. Let $\mu = (1/d) \sum_{i=1}^d x_i$. The layer norm $\mathbf y \in \mathbb D_p^m$ of a vector $\mathbf x \in \mathbb D_p^m$ is computed, for scalars $a,b \in \mathbb D_p$,
\begin{align*}
    \mathbf y &= a \left( \frac{\mathbf x - \mu}{\norm{\mathbf x - \mu}} \right) + b .
\end{align*}
\begin{lemma} \label{lem:layer-norm}
If $a,b$ are \logtimeunif, the layer norm over a vector of size $m$ can be computed by a \logtimeunif threshold circuit family of constant depth and size polynomial in $m$.
\end{lemma}
\begin{proof}
First compute $m$ using summation over the constant term $1$ from $1$ to $m$. This summation can be computed by a \logtimeunif constant-depth threshold circuit family of size polynomial in $m$. Then compute the sum over $\mathbf x$ using a similar circuit, and divide them to get $\mu$, using the fact that integer division is in \logtimeunif $\TC^0$ \citep{hesse2001division}. We can then compute $\mathbf x - \mu$ in \logtimeunif $\TC^0$.

At this point, we can compute $\norm{\mathbf x - \mu}$ in \logtimeunif $\TC^0$ \citep{hunter2010computing}, then divide each $\mathbf x - \mu$ by the norm in \logtimeunif $\TC^0$, and then apply the final affine transformation in \logtimeunif $\TC^0$ (\Cref{lem:affine}). Thus, computing layer norm is in \logtimeunif $\TC^0$.
\end{proof}

\section{Arithmetic Complexity}

\begin{lemma} \label{lem:division}
Given an $m$-bit integer $a$ and $n$-bit integer $b$, we can compute the quotient $\floor{a / b}$ and remainder $a \bmod b$ in time $\O(mn)$.
\end{lemma}

\begin{proof}
Let $D(m, n)$ and $M(m, n)$ denote, respectively, the time complexity of dividing and multiplying an $m$-bit integer by an $n$-bit integer.
\citet{brent2010modern} give the following fact: $D(m + n, n) \leq \O(M(m, n))$. With the goal of analyzing $D(m, n)$, we apply this as follows:
\begin{align*}
    D(m, n)
    &\leq D(m + n, n) \\
    &\leq \O(M(m, n)) \\
    &\leq \O(mn) . \qedhere
\end{align*}
\end{proof}

Applying \Cref{lem:division} when $a$ has size $\O(\log n)$ and $b$ has size $\O(1)$ says that we can do division in time $\O(\log n)$.

\end{document}